\documentclass[runningheads]{llncs}
\usepackage[T1]{fontenc}
\usepackage{graphicx}
\usepackage{amssymb}
\usepackage{amsmath}
\usepackage{amssymb}   
\usepackage{pgfplots}
\pgfplotsset{compat=1.18}
\usepackage{booktabs}
\usepackage{array}
\usepackage{algorithm}
\usepackage{algpseudocode}
\newcommand{\ptp}{\pi_{\mathrm{PTP}}}
\usepackage{xcolor}
\newcommand{\ML}[1]{\textcolor{red}{ML: #1}}
\renewcommand{\AA}[1]{\textcolor{blue}{AA: #1}}

\newcommand{\hpi}{\pi_{\mathrm{h}}}
\usepackage[nobreak]{mdframed}
\usepackage{graphicx}
\usepackage{subcaption}
\usepackage{subcaption} 
\usepackage[most]{tcolorbox} 
\definecolor{theorembar}{RGB}{33,49,74}   
\definecolor{theorembody}{RGB}{236,240,246} 
\definecolor{theoremborder}{RGB}{33,49,74}
\usepackage{pgfplots}
\pgfplotsset{compat=1.18} 

\usetikzlibrary{spy}
\newtcbtheorem[no counter]{proofblock}{Proof}{
    enhanced,
    sharp corners,                          
    boxrule=0.5pt,                          
    colback=blue!5!white,                   
    colframe=blue!75!black,                 
    fonttitle=\bfseries,                    
    coltitle=white,                         
    attach boxed title to top left={        
        yshifttext=-1mm, 
        xshift=2mm 
    },
    boxed title style={
        sharp corners, 
        size=small, 
        colback=blue!75!black               
    }
}{pb}

\definecolor{theorembar}{RGB}{33,49,74}    
\definecolor{theorembody}{RGB}{236,240,246} 
\definecolor{theoremborder}{RGB}{33,49,74}  

\newtcbtheorem[number within=section]{mytheorem}{Theorem}{
    enhanced,
    sharp corners,
    boxrule=0.5pt,
    colback=theorembody,             
    colframe=theoremborder,         
    fonttitle=\bfseries\sffamily,   
    coltitle=white,                 
    attach boxed title to top left={yshifttext=-1mm, xshift=2mm},
    boxed title style={
        sharp corners, 
        colback=theorembar          
    },
    separator sign={:},
}{th}

\newtcbtheorem{myproposition}{Proposition}{
    enhanced,
    sharp corners,
    boxrule=0.5pt,
    colback=theorembody,
    colframe=theoremborder,
    fonttitle=\bfseries\sffamily,
    coltitle=white,
    attach boxed title to top left={yshifttext=-1mm, xshift=2mm},
    boxed title style={
        sharp corners, 
        colback=theorembar
    },
    separator sign={:},
}{prop}

\newtcbtheorem{myobservation}{Observation}{
    enhanced,
    sharp corners,
    boxrule=0.5pt,
    colback=theorembody,
    colframe=theoremborder,
    fonttitle=\bfseries\sffamily,
    coltitle=white,
    attach boxed title to top left={yshifttext=-1mm, xshift=2mm},
    boxed title style={
        sharp corners, 
        colback=theorembar
    },
    separator sign={:},
}{observ}

\definecolor{algobar}{RGB}{0, 104, 132}
\definecolor{algobody}{RGB}{235, 247, 247}
\newtcbtheorem{myalgorithm}{Algorithm}{
    enhanced,
    sharp corners,
    boxrule=0.5pt,
    colback=algobody,             
    colframe=algobar,              
    fonttitle=\bfseries\sffamily,
    coltitle=white,
    attach boxed title to top left={yshifttext=-1mm, xshift=2mm},
    boxed title style={
        sharp corners, 
        colback=algobar            
    },
    separator sign={:},
}{algor}

\begin{document}
%
\title{Active Reward Machine Inference From Raw State Trajectories}
\titlerunning{Active Reward Machine Inference}
%
\author{Mohamad Louai Shehab\inst{1}\orcidID{0009-0009-7179-3881} \and Antoine Aspeel\inst{2}\orcidID{0000-0003-3011-7122} \and
Necmiye Ozay\inst{1,3}\orcidID{0000-0002-5552-4392}}
\authorrunning{M.L. Shehab et al.}
%
\institute{Robotics Department, University of Michigan, Ann Arbor, MI \and Independent Researcher \and Electrical Engineering and Computer Science Department, University of Michigan, Ann Arbor, MI}

\maketitle        

\begin{abstract}
Reward machines are automaton-like structures that capture the memory required to accomplish a multi-stage task. When combined with reinforcement learning or optimal control methods, they can be used to synthesize robot policies to achieve such tasks. However, specifying a reward machine by hand, including a labeling function capturing high-level features that the decisions are based on, can be a daunting task. This paper deals with the problem of learning reward machines directly from raw state and policy information. As opposed to existing works, we assume no access to observations of rewards, labels, or machine nodes, and show what trajectory data is sufficient for learning the reward machine in this information-scarce regime. We then extend the result to an active learning setting where we incrementally query trajectory extensions to improve data (and indirectly computational) efficiency. Results are demonstrated with several grid world examples.

\end{abstract}

\setcounter{footnote}{0}
\section{Introduction}

Multi-stage tasks are ubiquitous in robotics applications, as real-world objectives are rarely achieved through a single atomic action but instead require the coordinated execution of sequential and interdependent subtasks \cite{Sutton1998Options,Barto2003HRL,KressGazit2018TL}. These robots have to operate under temporally extended objectives in which task success depends on satisfying intermediate goals in a specific order, motivating the use of hierarchical and modular task representations \cite{Andreas2017Modular}. Reward machines provide a principled and expressive formalism for representing such multi-stage task structure by encoding task progress as a finite-state automaton whose transitions are triggered by high-level events or propositions observed during execution \cite{ToroIcarte2018RM,ToroIcarte2019LMRL,camacho2021reward,defazio2024learning,baert2025reward}. 

While Reward Machines (RMs) offer a powerful framework for structured task execution, their utility is often bottlenecked by the requirement for human experts to manually specify the underlying automaton. In complex, real-world environments, defining the exact logical transitions and propositional triggers for a task is not only labor-intensive but also prone to specification errors that can lead to unintended robotic behaviors \cite{Amodei2016Safety}. Consequently, there is a growing imperative to develop algorithms capable of learning RMs directly from experience \cite{ToroIcarte2019LMRL,Xu2020DeepRML,shehab2025learning}. This automated RM inference allows a robot to autonomously discover the latent logical structure of a task, identifying the ``hidden'' stages that define successful execution without explicit human oversight. 

Existing literature on learning reward machines generally falls into three categories: those assuming ground-truth labels for states or rewards to find consistent automata \cite{araki2019learning,xu2020joint,icarte2023learning,hu2024reinforcement,abate2023learning}, those utilizing active learning or $L^*$ oracles to query membership and conjectures \cite{angluin1987learning,xu2021active,memarian2020active}, and those combining automata synthesis with reinforcement learning in interactive environments \cite{hasanbeig2024symbolic,hasanbeig2021deepsynth,furelos2020induction}. While some approaches rely on observing demonstrations \cite{camacho2021reward,baert2024reward}, they are often restricted to single-stage tasks where the RM serves primarily for reward shaping rather than complex temporal logic. Other works learn reward machines for multi-staged tasks purely from demonstrations \cite{shehab2025learning}.  A critical bottleneck in these works is the reliance on a predefined labeling function that maps low-level states to high-level propositions. Recent efforts have begun addressing labeling function ambiguity by accounting for noise and uncertainty in label assignments \cite{li2024reward,parac2024learning,verginis2024joint}; however, they still assume the existence of a noisy labeling function. In contrast, our framework is the first to learn both the labeling function and the reward machine from scratch using only state-based traces, eliminating the need for any prior symbolic knowledge or predefined event detectors.\\

\textbf{Notation:} For a finite set $X$, we denote by $|X|$ its cardinality. The set of all probability distributions over $X$ is denoted by $\Delta(X)$. The set of all finite sequences with elements in $X$ is denoted by $X^*$. For two sets $X$ and $Y$, their Cartesian product is $X \times Y$, and the Cartesian power of $X$ of order $n$ is denoted by $X^n$. The set of real numbers is denoted by $\mathbb{R}$. Logical conjunction and disjunction are denoted by $\wedge$ and $\vee$, respectively. The expectation operator is denoted by $\mathbb{E}$.

\section{Preliminaries}

\subsection{Markov decision processes and reward machines}

A \emph{Markov Decision Process} (MDP) is a tuple
$$
\mathcal{M} = (\mathcal{S},\mathcal{A}, \mathcal{P},\mu_0, \gamma, r),
$$
where $\mathcal{S}$ is the finite set of states, $\mathcal{A}$ is the finite set of actions, $\mathcal{P}:\mathcal{S}\times \mathcal{A} \to \Delta(\mathcal{S})$ is the Markovian transition kernel, $\mu_0 \in \Delta(\mathcal{S})$ is the initial state distribution, $\gamma \in [0,1)$ is the discount factor, and $r:\mathcal{S}\times \mathcal{A}\times \mathcal{S}\to \mathbb{R}$ is the reward function. We refer to a MDP without the reward $r$ as an \emph{MDP model}, and denote it by $\mathcal{M}\setminus r$.

A \emph{Reward Machine} (RM) is a tuple
$$
\mathcal{R}=(\mathcal{U}, u_I, \mathrm{AP}, \delta_{\mathbf{u}}, \delta_{\mathbf{r}}),
$$
where $\mathcal{U}$ is the finite set of nodes, $u_I\in\mathcal{U}$ is the initial node, $\mathrm{AP}$ is the set of atomic propositions (also called input alphabet), $\delta_{\mathbf{u}}: \mathcal{U} \times \mathrm{AP}\to \mathcal{U}$ is the (deterministic) transition function, and $\delta_{\mathbf{r}}: \mathcal{U}\times \mathrm{AP} \to \mathbb{R}$ is the output function. A reward machine without its output function is named a \emph{reward machine model}. We extend the definition of the transition function to define $\delta_{\mathbf{u}}^*: \mathcal{U} \times (\mathrm{AP})^* \to \mathcal{U}$ as $\delta_{\mathbf{u}}^*(u, l_0 , \cdots, l_k) = \delta_{\mathbf{u}}(\cdots(\delta_{\mathbf{u}}(\delta_{\mathbf{u}}(u, l_0),l_1), \cdots, l_k)$. 

Finally, a \emph{labeling function} (compatible with an MDP $\mathcal{M}$ and a RM $\mathcal{R}$) is a function $L:\mathcal{S}\to\mathrm{AP}$ which associates an atomic proposition of a RM to each state of an MDP.

It is common to introduce the notion of labeled MDP as a pair formed by an MDP and a compatible labeling function. However, in the problem we are interested in, both the RM and the labeling function are unknown. Consequently, it will make our notation simpler to consider a labeled RM instead. Formally, a \emph{labeled RM} refers to a pair composed by a reward machine and a (compatible) labeling function $\mathcal{R}_L=(\mathcal{R},\mathcal{S},L)$. A \emph{labeled RM model} is a labeled RM without its output function $\delta_\mathbf{r}$ and is denoted by $\mathcal{G}$. In the next section, we show how a compatible labeling function allows to ``connect'' an MDP model with a RM.

\subsection{Product MDP} \label{sec:productMDP}
An MDP model $\mathcal{M}\setminus r$ together with a reward machine $\mathcal{R}$ and a compatible labeling function $L$  allow to define a \emph{product MDP}
$$
\mathcal{M}_\mathrm{Prod} = (\mathcal{S}', \mathcal{A}', \mathcal{P}', \mu_0^\prime, \gamma^\prime, r'),
$$
where $\mathcal{S}' = \mathcal{S}\times \mathcal{U}$, $\mathcal{A}' = \mathcal{A}$, $\mathcal{P}'(s',u'| s,u,a) = \mathcal{P}(s'|s,a) \mathbf{1}(u' = \delta_\textbf{u}(u,L(s')))$, $\gamma^\prime = \gamma$, $\mu_0^\prime \in \Delta(\mathcal{S}\times \mathcal{U})$ with $\mu_0^\prime(s,u) = \mu_0(s)\mathbf{1}(u = u_I)$ and $r'(s,u,a,s',u') = \delta_\mathbf{r}(u,L(s^\prime))$, where $\mathbf{1}(p)$ is one if $p$ is true and zero otherwise. To make the notation compact, we denote the product state by $\bar s = (s,u)$.

A \emph{trajectory} of the product MDP $\mathcal{M}_\mathrm{Prod}$ is a sequence $(\bar s_{\emptyset}, a_{\emptyset}, \bar s_0, a_0, \bar s_1, a_1,\cdots)$, where $\bar s_{\emptyset} = (\emptyset, u_I)$ and $a_{\emptyset}= \emptyset$. An initial state $s_0$ is sampled from $\mu_0$. The introduction of $\bar s_{\emptyset}$ and $a_\emptyset$ at the start of the trajectory is to ensure that $s_0$ induces a transition in the reward machine. The reward machine thus transitions to $u_0 = \delta_\textbf{u}(u_I, L(s_0))$. The agent then takes action $a_0$ and transitions to $s_1$. Similarly, the reward machine transitions to $u_1 = \delta_\textbf{u}(u_0, L(s_1))$. The same procedure continues infinitely. We consider the product policy $\pi_{\mathrm{Prod}}:\mathrm{Dom_{Prod}} \to \Delta(\mathcal{A})$ where $\mathrm{Dom_{Prod}}\subseteq \mathcal{S}\times\mathcal{U}$ is the set of accessible $(s,u)$ pairs in the product MDP. This policy is a function that describes an agent’s behavior by specifying an action distribution at each state.  We consider the Maximum Entropy Reinforcement Learning (MaxEntRL) objective given by:
\begin{equation}\label{eq:max_ent_obj}
    J_{\mathrm{MaxEnt}}(\pi;r') = \mathbb{E}^{\pi}_{\mu_0}[\sum\limits_{t=0}^{\infty} \gamma^t \biggl(  r^\prime (\bar s_t,a_t, \bar s_{t+1}) + \lambda \mathcal{H}(\pi(.|\bar s_t)) \biggr)],
 \end{equation}
where $\lambda > 0$ is a regularization parameter, and $\mathcal{H}(\pi(.|\bar{s})) = -\sum\limits_{a\in \mathcal{A}} \pi(a|\bar{s})\log(\pi(a|\bar{s}))$ is the entropy of the policy $\pi$. The expectation is with respect to the probability distribution $\mathbb{P}^\pi_{\mu_0}$, the induced distribution over infinite trajectories following $\pi$, $\mu_0$, and the Markovian transition kernel $\mathcal{P}^\prime$ \cite{ziebart2008maximum}. The optimal policy $\pi_{\mathrm{Prod}}^*$, corresponding to a reward function $r'$, is the maximizer of (\ref{eq:max_ent_obj}), i.e.,
\begin{equation}\label{eq:opt_prob}
    \pi_{\mathrm{Prod}}^* = \arg \max\limits_{\pi} J_{\mathrm{MaxEnt}}(\pi;r').
\end{equation}
As shown in \cite{ziebart2008maximum}, the maximizer is unique and this policy is well defined. 

\section{Problem Statement}

We investigate the problem of learning a reward machine that makes a policy optimal, without assuming that the rewards, the machine nodes, or the atomic propositions are observed. This problem is ill-posed since several reward machines could make a policy optimal. To write this problem formally, we first introduce the notion of history policy (which captures the information available for solving this learning problem), and then the notion of policy equivalence (which characterizes equivalent solutions of this learning problem).

\subsection{History Policy and Reward Machine Equivalence}
Consider an MDP model, a RM, and a compatible labeling function. As shown in Section~\ref{sec:productMDP}, this allows to define the product MDP and the corresponding optimal product policy. The \emph{history policy}, denoted $\hpi: \mathrm{Dom_h} \to \Delta(\mathcal{A})$ is defined as
\begin{align} \label{eq:induced}
\hpi(a| s,\tau) = \pi_{\mathrm{Prod}}(a| s,\delta_u^\star(u_I,L(\tau))).
\end{align}
The domain $\mathrm{Dom_h}\subseteq\mathcal{S} \times \mathcal{S}^*$ is the largest set for which the right-hand side of \eqref{eq:induced} is well defined. In \eqref{eq:induced}, $\tau \in \mathcal{S}^*$ is a trajectory of states leading up to the current state $s\in\mathcal{S}$. While $\pi_{\mathrm{Prod}}$, $\delta_u$, and $L$ are not known, the history policy $\hpi$ is assumed to be available\footnote{While we make this assumption for simplicity, in practice, instead of $\hpi$, one has access to state-action trajectories of an agent implementing $\hpi$. Then, $\hpi$ can be estimated by a sample average. A discussion on the effects of such estimation can be found in \cite{shehab2025learning}.}. In that sense, the history policy serves as a state-only representation of the product policy, capturing the agent's behavior through the sequence of observable MDP states. We say that the product policy $\pi_{\mathrm{Prod}}$ \emph{induces} $\hpi$.

The history policy can take an arbitrarily long trajectory $\tau$ as argument. In contrast, we define the \emph{depth-$l$ restriction} of the history policy, denoted as $\hpi^{l}$, by restricting its domain to trajectories of length at most $l$, i.e., the domain of $\hpi^{l}$ is $\mathrm{Dom_h}\cap \left(\mathcal{S} \times\cup_{j=1}^{l} \mathcal{S}^j\right)$. 

As mentioned before, the inverse reinforcement learning problem we are interested in can have multiple solutions. This is captured by the following definition.

\begin{definition}
    Two labeled reward machines are \textbf{policy-equivalent} with respect to an MDP model if the optimal product policies for each of the labeled reward machines induce the same history policy. Among all the labeled reward machines that are policy equivalent with respect to an MDP model, we define a \textbf{minimal} reward machine as one with the fewest number of nodes.
\end{definition}

\subsection{Formal problem statement}\label{sec:problem}
We now have all the ingredients to formalize the problems we are interested in. For an MDP model and labeled RM, consider the induced optimal history policy. Knowing the MDP model and (a depth-$l^*$ restriction of) the history policy, is it possible to recover a labeled RM that is policy equivalent to the true one? This research question can be divided in the following:
\begin{enumerate}
\item[(P1)] Does there always exist a depth $l^*$ such that, given the MDP model $\mathcal{M}$, an upper bound $u_{\mathrm{max}}$ on the number of nodes of the underlying reward machine and the depth-$l^*$ restriction $\hpi^l$ of the true history policy, it is possible to learn a labeled reward machine that is policy-equivalent to the underlying one?
\item[(P2)] If $l^*$ in problem (P1) exists, find a minimal labeled reward machine that is policy-equivalent to the underlying one.
\end{enumerate}

\section{Methodology}
The problem of learning a reward machine (RM) directly from policies when the labels are known has previously been addressed in~\cite{shehab2024learning} with a two step process: 1) a Boolean Satisfiability (SAT) problem \cite{cook1971complexity,biere2009handbook} to learn a RM model; 2) a structured IRL problem that uses the product of the labeled RM model from step 1 and the MDP to recover the reward function. We present in this section how to extend this framework, and in particular step 1, to the more general and challenging setting in which the labeling function is unknown and must be learned jointly with the reward machine model.

\subsection{Learning a Labeled Reward Machine}

The unknown quantities that we aim to learn are the RM transition function $\delta_\mathbf{u}$, the labeling function $L$, and the output function $\delta_\mathbf{r}$. Learning $\delta_\mathbf{u}$, and $L$ corresponds to the generalization of step 1 of the process described above. Learning an output function, i.e., a reward function for the product, consistent with a policy corresponds to step 2 and has been addressed in prior works~\cite{cao2021identifiability,shehab2024learning}. Therefore, this paper focuses on inferring $\delta_\mathbf{u}$ and $L$. In this section, we present a SAT problem allowing to learn $\delta_\mathbf{u}$ and $L$, i.e., a labeled reward machine model denoted $\hat{\mathcal{G}}$. Without loss of generality, let us write $\mathcal{S}=\{1,\dots,|\mathcal{S}|\}$, $\mathcal{U}=\{1,\dots,|\mathcal{U}|\}$, $\mathrm{AP}=\{1,\dots,|\mathrm{AP}|\}$, and consider $u_I=1$ and $L(1)=1$.

The transition function $\delta_\mathbf{u}$ and the labeling function $L$ are encoded by binary values as follows:
\begin{equation}\label{eq:encoding}
b_{jpi} =
\begin{cases}
1 & \text{if } \delta_\mathbf{u}(i,p)=j, \\
0 & \text{otherwise},
\end{cases} \ \ \text{ and }\ \ \mathbf{L}_{pk}= \begin{cases}
1 & \text{if } L(k)=p, \\
0 & \text{otherwise},
\end{cases}
\end{equation}
where $i,j=1,\dots,|\mathcal{U}|$, $p=1,\dots,|\mathrm{AP}|$, and $k=1,\dots,|\mathcal{S}|$. The core constraints in the SAT formulation arise from negative examples, which follow from the following lemma.
\begin{lemma}\label{lemma:negativeExample}
Let $\tau,\tau'\in\mathcal{S}^*$ be two state trajectories.  If $\hpi(a|s,\tau)\neq\hpi(a|s,\tau')$ for some $(s,a)\in\mathcal{S}\times\mathcal{A}$, then $\delta_\mathbf{u}^*(u_I,L(\tau))\neq\delta_\mathbf{u}(u_I,L(\tau'))$.
\end{lemma}
\begin{proof}
It follows directly from the definition of the history policy (equation~\eqref{eq:induced}).
\end{proof}
A pair of state trajectories $\{\tau,\tau'\}$ that satisfies the assumption of Lemma~\ref{lemma:negativeExample} is called a \emph{negative example}, meaning the atomic proposition trajectories $L(\tau)$ and $L(\tau')$ should lead to different reward machine nodes starting at $u_I$. The following set collects all negative examples of length at most $l$:
\begin{equation}
\mathcal{E}^-_l =
\left\{
\{\tau,\tau'\} \;\middle|\;
\hpi^l(a|s,\tau) \neq \hpi^l(a|s,\tau')
\text{ for some } (s,a)\in\mathcal{S}\times\mathcal{A}
\right\}.
\end{equation}
Note that $\tau$ and $\tau'$ may have different lengths (both not larger than $l$).
For each pair $\{\tau,\tau'\}\in\mathcal{E}^-_l$, Lemma~\ref{lemma:negativeExample} imposes some constraints on $\delta_\mathbf{u}$ and $L$. To write these constraints via our binary encoding we need to introduce some notation. First, let us define the binary matrices $(B_p)_{ji} = b_{jpi}$ and note that they satisfy $(B_p)_{ji}=1$ if and only if $\delta_\mathbf{u}(i,p)=j$. {In words, $B_p$ represents the transitions in the RM for a given atomic proposition $p$.} Next, for a state $k\in\mathcal{S}$, consider the matrix
\begin{equation}
M_k =
(B_1 \wedge^\star \mathbf{L}_{1,k})
\;\vee\;
\cdots
\;\vee\;
(B_{|\mathrm{AP}|} \wedge^\star \mathbf{L}_{|\mathrm{AP}|,k}),
\end{equation}
where $\wedge^\star$ denotes point-wise conjunction between a Boolean matrix and a Boolean scalar. This binary matrix satisfies $(M_k)_{ji}=1$ if and only if $\delta_\mathbf{u}(i,L(k))=j$. In words, $M_k$ represents the transitions in the RM for a given MDP state $k$. Finally, for a state trajectory $\tau=(k_1,\dots,k_t)\in\mathcal{S}^t$, consider the binary vector
$$
v_\tau=M_{k_t}M_{k_{t-1}}\dots M_{k_1} \begin{bmatrix}1& 0 & \cdots & 0\end{bmatrix}^\top,
$$
which satisfies $(v_\tau)_i=1$ if and only if $\delta^*(u_I,L(\tau))=i$ (let us remind that we assumed $u_I=1$). {Here, $v_\tau$ represents the node at which the RM will be after emitting the labels of the state trajectory $\tau$.} For a pair of negative examples $\{\tau,\tau'\}\in\mathcal{E}^-_l$, the condition $\delta_\mathbf{u}^*(u_I,L(\tau))\neq\delta_\mathbf{u}^*(u_I,L(\tau'))$ given by Lemma~\ref{lemma:negativeExample} can be written $v_\tau\neq v_{\tau'}$.

Overall, the SAT problem that encodes the learning of $\delta_\mathbf{u}$ and $L$ is the following:
\begin{problem} \label{prob:SAT}
For a fixed $l$, find $b_{jpi}$ and $\mathbf{L}_{pk}$ for $i,j=1,\dots,|\mathcal{U}|$, $p=1,\dots,|\mathrm{AP}|$, and $k=1,\dots,|\mathcal{S}|$ such that:
\begin{subequations}
\begin{align}
\sum_j b_{jpi} &= 1 && (\delta_\mathbf{u} \text{ is a function}) \label{eq:SAT:delt_u} \\
\sum_p \mathbf{L}_{pk} &= 1 && (L \text{ is a function}) \label{eq:SAT:L}\\
\mathbf{L}_{1,1} &= 1 && (\text{Anchoring: } L(1)=1) \label{eq:SAT:anchoring}\\
\forall\{\tau,\tau'\}\in\mathcal{E}_l^-:\ v_\tau &\neq v_{\tau'} && (\text{Compatibility with negative examples}) \label{eq:SAT:negativeExamples}\\
b_{jpi}=1 \Rightarrow b_{jpj} &= 1 && (\text{Non-stuttering, optional constraint}) \label{eq:SAT:stuttering}
\end{align}
\end{subequations}
where each constraint must hold for all free indices.
\end{problem}

Constraints~\eqref{eq:SAT:delt_u} and~\eqref{eq:SAT:L} ensure that $\delta_\mathbf{u}$ and $L$ are well defined, respectively. Constraint~\eqref{eq:SAT:anchoring} removes some solutions which are equivalent up to permutation. Constraint~\eqref{eq:SAT:negativeExamples} enforces the conditions given by Lemma~\ref{lemma:negativeExample}. Finally, constraint~\eqref{eq:SAT:stuttering} allows to enforce some prior knowledge on the reward machine, when such information exits. More precisely, it enforces
\begin{equation*}
\forall i,j,p:\delta_\mathbf{u}(i,p)=j\Rightarrow\delta_\mathbf{u}(j,p)=j
\end{equation*}
which holds if the underlying task is multi-stage and duration-insensitive (i.e., stutter-invariant). This constraint prevents repeated self-transitions under the same proposition and enables trace compression \cite{icarte2023learning,shehab2025learning}. Note that in Problem~\ref{prob:SAT}, it is enough to know an upper bound $u_\mathrm{max}$ on $|\mathcal{U}|$. The same is true for $|\mathrm{AP}|$, but one can always choose $|\mathrm{AP}|=|\mathcal{S}|$. Indeed, if there are more atomic propositions than states, one can consider only the atomic propositions in $L(\mathcal{S})$ whose cardinality is at most $|\mathcal{S}|$. Next, we present the main theoretical result of this section. Proposition~\ref{prop:sufficientDepth} below specifies the required $l^*$ from Section~\ref{sec:problem}.

\begin{myproposition}{Sufficient Depth}{sufficientDepth}
Given an MDP model, an upper bound $u_{\mathrm{max}}$ on the number of nodes of the underlying reward machine and the depth-$l^*$ restriction $\hpi^{l^*}$ of some history policy $\hpi$, where $l^* = |\mathcal{S}|u_{\mathrm{max}}^2$, Problem~\ref{prob:SAT} is satisfiable with $l=l^*$ if and only if it is satisfiable for all $l>l^*$.
\end{myproposition}

\subsection{Proof of Proposition~\ref{prop:sufficientDepth}}
In order to prove Proposition~\ref{prop:sufficientDepth}, let us first introduce the notion of synchronized labeled reward machine model. It allows to run two labeled reward machine models in parallel. It will be used to compare the ground truth labeled reward machine with the learned one.

\begin{definition}\label{def:dsync}
    Let $\mathcal{G}_1 = (\mathcal{U}_1, u_I^1, \mathrm{AP}^1, \delta_{\mathbf{u}}^1, \mathcal{S}, L^1),\  \mathcal{G}_2 = (\mathcal{U}_2, u_I^2, \mathrm{AP}^2, \delta_{\mathbf{u}}^2,\mathcal{S}, L^2)$ be two labeled reward machine models with the domains of $L^1$ and $L^2$ being a same set $\mathcal{S}$. The \textbf{synchronized labeled reward machine model} is the labeled reward machine model defined as follows:
    \begin{align*}
    \mathcal{G}^{\mathrm{sync}} &= (\mathcal{U}^{\mathrm{sync}}, u_I^{\mathrm{sync}}, \mathrm{AP}^{\mathrm{sync}}, \delta_{\mathbf{u}}^{\mathrm{sync}}, \mathcal{S}, L^{\mathrm{sync}}) \\
    \mathcal{U}^{\mathrm{sync}} &= \mathcal{U}_1\times \mathcal{U}_2, \\
    u_I^{\mathrm{sync}} &= (u_I^1,u_I^2), \\
    \mathrm{AP}^{\mathrm{sync}} &= \mathrm{AP}^1\times\mathrm{AP}^2\\
    \delta_{\mathbf{u}}^{\mathrm{sync}}((u^1,u^2), (l^1,l^2)) &= (\delta_{\mathbf{u}}^1(u^1,l^1), \delta_{\mathbf{u}}^2(u^2,l^2))\\
    L^{\mathrm{sync}}(s) &= (L^1(s),L^2(s)).
\end{align*}
\end{definition}
The following definition introduces cycles in a product MDP. The core of the proof relies on removing cycles in the synchronized product MDP model.

\begin{definition}
Let $\mathcal{M}\setminus r$ be an MDP model and $\mathcal{G}$ be a (compatible) labeled reward machine model. Let $\mathcal{M}_{\mathrm{Prod}}$ be the corresponding product MDP model. Given a state trajectory $\tau = (s_1,s_2,\cdots, s_t)\in\mathcal{S}^*$, we say that a subsequence $s_{i:j}$ of $\tau$ is a \textbf{cycle} in $\mathcal{M}_{\mathrm{Prod}}$ if $s_i = s_j$ and $\delta_{\textbf{u}}^*(u_I,L(s_{:i})) = \delta_{\textbf{u}}^*(u_I,L(s_{:j}))$.
\end{definition}

\begin{proof}[of Proposition~\ref{prop:sufficientDepth}]

We will refer to Problem~\ref{prob:SAT} with a given $l$ by $\text{SAT}_l$. Let $j>l^*$ be a natural number.

First, let us prove the ``if'' direction. Assume that $\text{SAT}_j$ has a solution. Then, it satisfies constraint~\eqref{eq:SAT:negativeExamples} for all $\{\tau,\tau'\}\in\mathcal{E}^-_j$. But since $l^*<j$, $\mathcal{E}^-_{l^*}\subseteq\mathcal{E}^-_j$. Since removing constraints can not make a solution infeasible, the solution of $\text{SAT}_j$ is also a solution of $\text{SAT}_{l^*}$.

Second, let us prove the ``only if'' direction. By contradiction, assume that $\text{SAT}_{l^*}$ has a solution which is not a solution to $\text{SAT}_j$. This solution defines functions $\hat{\delta}_\mathbf{u}$ and $\hat{L}$ through the binary encoding~\eqref{eq:encoding}. Since this is not a solution to $\text{SAT}_j$, there exists a pair of negative examples $\{\tau,\tau'\}\in\mathcal{E}^-_j$ which does not satisfy condition~\eqref{eq:SAT:negativeExamples}. That is, $\tau,\tau'$ satisfy
\begin{equation}\label{eq:proof:negativeExample}
\hpi(a|s,\tau) \neq \hpi(a|s,\tau')
\end{equation}
for some $(s,a)\in\mathcal{S}\times\mathcal{A}$, and
\begin{equation}\label{eq:proof:contradiction}
\hat{\delta}_\mathbf{u}(u_I,\hat{L}(\tau))=\hat{\delta}_\mathbf{u}(u_I,\hat{L}(\tau')).
\end{equation}
Denote by $\mathcal{G}^\mathrm{sync}$ the synchronized product between the true labeled reward machine model $\mathcal{G}=(\mathcal{U},u_I,\mathrm{AP},\delta_\mathbf{u},\mathcal{S},L)$ and the learned one $\hat{\mathcal{G}}$. Consider the product MDP model $\mathcal{M}_{prod,m}$ obtained from the MDP model and the synchronized labeled RM model. A state in $\mathcal{M}_{prod,m}$ is a tuple $(s,u,\hat{u})\in\mathcal{S}\times\mathcal{U}\times\hat{\mathcal{U}}$, which shows that $\mathcal{M}_{prod,m}$ contains $|\mathcal{S}| |\mathcal{U}| |\hat{\mathcal{U}}|\leq |\mathcal{S}|u_\mathrm{max}^2=l^*$ states.

Now, consider the trajectory $\bar{\tau}$ (resp. $\bar{\tau}'$) obtained by removing cycles of $\tau$ (resp. $\tau'$) in $\mathcal{M}_{prod,m}$. Since $\mathcal{M}_{prod,m}$ has at most $l^*$ states, this can be repeated until $|\bar{\tau}|\leq l^*$ (resp. $|\bar{\tau}'|\leq l^*$). Since only cycles have been removed, the corresponding nodes in the synchronized labeled RM model stay unchanged, i.e.,
$$
\delta^{\mathrm{sync},*}_\mathbf{u}(u_I^\mathrm{sync},L^\mathrm{sync}(\tau)) = \delta^{\mathrm{sync},*}_\mathbf{u}(u_I^\mathrm{sync},L^\mathrm{sync}(\bar{\tau})),
$$
and similarly for $\tau'$. By definition of the synchronized labeled RM model, this gives
\begin{align}
\delta^*_\mathbf{u}(u_I,L(\tau))&=\delta^*_\mathbf{u}(u_I,L(\bar{\tau})), \label{eq:proof:true} \\
\hat{\delta}^*_\mathbf{u}(u_I,\hat{L}(\tau))&=\hat{\delta}^*_\mathbf{u}(u_I,\hat{L}(\bar{\tau})), \label{eq:proof:learn}
\end{align}
and similarly for $\tau'$.

It follows from \eqref{eq:proof:true} and \eqref{eq:induced} that $\hpi(a|s,\tau)=\hpi(a|s,\bar{\tau})$, and similarly for $\tau'$. Consequently, \eqref{eq:proof:negativeExample} implies $\hpi(a|s,\bar{\tau}) \neq \hpi(a|s,\bar{\tau}')$. That is, the pair $\{\bar{\tau},\bar{\tau}'\}$ is a negative example, i.e., $\{\bar{\tau},\bar{\tau}'\}\in\mathcal{E}^-_{l^*}$. In addition, it follows from \eqref{eq:proof:learn} and \eqref{eq:proof:contradiction} that $\hat{\delta}_\mathbf{u}(u_I,\hat{L}(\bar{\tau}))=\hat{\delta}_\mathbf{u}(u_I,\hat{L}(\bar{\tau}'))$. Overall, we have shown that the pair $\{\bar{\tau},\bar{\tau}'\}\in\mathcal{E}^-_{l^*}$ contradicts Lemma~\ref{lemma:negativeExample} which is encoded as constraint~\eqref{eq:SAT:negativeExamples}. Therefore, the solution to $\text{SAT}_{l^*}$ does not satisfy constraint~\eqref{eq:SAT:negativeExamples}, a contradiction.
\end{proof}

\subsection{Active Extension of the History Policy}
One major bottleneck for solving Problem~\ref{prob:SAT} comes from encoding all the negative examples in (\ref{eq:SAT:negativeExamples}) given a depth-$l$ restriction of the history policy. Since the number of state trajectories for a standard stochastic MDP grows exponentially in the depth, representing (or storing) the history policy becomes increasingly infeasible, even for moderate depths and small state spaces. However, our key observation is that exhausting all the possible paths of the history policy is not necessary to shrink the solution set. We formalize this observation as follows.
\begin{myobservation}{Not All Trajectories are Created Equal}{obs_pi}
    Suppose that we solved the SAT problem with a depth-$l$ restriction of the history policy, which is less than the sufficient depth $l^*$, and obtained a solution set of candidate labeled reward machine models. Let $\tau$ be a length $(l+k)$ state trajectory, for any $k \geq 1$, and let $\tau_{:l}$ be the first $l$ states in $\tau$. Finally, let $\delta_{\mathbf{u}},L$ be a ground truth transition function and labeling function respectively. If $\delta^*_{\mathbf{u}}(u_I, L(\tau)) = \delta^*_{\mathbf{u}}(u_I, L(\tau_{:l}))$, then $\tau$ will not reduce the candidate solution set, as it cannot introduce any new negative examples.
\end{myobservation}
While we do not know $\delta_{\mathbf{u}}$ or $L$, the above observation aims at highlighting that many paths in the history policy are redundant when it comes to shrinking the candidate solution set. Hence, this motivates designing an active extension algorithm, which reduces the memory and computation requirements of fully extending the history policy. Our strategy adopts a \emph{volume-removal} approach to active learning, where queries are selected to significantly reduce the number of hypotheses consistent with current observations. Similar volume-reduction techniques have proven effective in preference-based reward learning~\cite{sadigh2017active,biyik2018batch,biyik2019asking,dasgupta2004analysis}.

In our setting, we query trajectory extensions (histories) that are expected to most rapidly eliminate candidate labeled reward machine models consistent with the current depth history policy. For a given depth $l < l^*$, let the set of feasible solutions of Problem~\ref{prob:SAT} be denoted $\mathcal{P}_{\mathrm{feasible}} = \{\hat{\mathcal{G}_i}\}_{i=1}^N$, where $N$ is the total number of feasible solutions. This depth $l$ represents the \emph{burn-in} cost in order to obtain a reasonably sized solution set. The key idea is to search for state trajectory pairs $\{\tau,\tau'\}$ for which half of the labeled reward machine models in the solution set end up in the same node, and the other half does not. If any pair $\{\tau,\tau'\}$ turns out to be an actual negative example when querying the extended ground truth history policy, then we have eliminated half of the reward machines in the solution set by just adding a single negative example. To formalize this, let $\mathfrak{B}$ be the query budget, which represents the maximum number of state trajectory pairs that we can query the history policy by. Our active learning algorithm runs as follows: 

\begin{myalgorithm}{Active Extension Algorithm}{alg_label}
\begin{enumerate}
\item \textbf{Initialize:} Start with an empty candidate set $\mathcal{C} = \emptyset.$
\item \textbf{Subsample:} Sample a subset of the feasible SAT solutions $\mathcal{P}_{\mathrm{active}} \triangleq \{\hat{\mathcal{G}}_i\}_{i=1}^{N_{\mathrm{active}}}$.
\item \textbf{Generate Candidates:} For each $\hat{\mathcal{G}}_i = (\hat{\mathcal{U}}, \hat u_I, {\mathrm{AP}^i}, \hat \delta_\mathbf{u}^i,\mathcal{S}, \hat L^i) \in \mathcal{P}_{\mathrm{active}}$, sample a random target node $u_{\mathrm{target}} \in \hat{\mathcal{U}}$ and use randomized DFS \cite{motwani1996randomized} to find trajectory pairs $\{\tau, \tau'\}$ of length $l+1$ ending in the same MDP state such that $\hat \delta_{\mathbf{u}}^{i,*}(\hat u_I, \hat L^i(\tau)) = \hat \delta_{\mathbf{u}}^{i,*}(\hat u_I, \hat L^i(\tau')) = u_{\mathrm{target}}$. Add these pairs to $\mathcal{C}$.

\item \textbf{Evaluate Quality:} For each $\{\tau, \tau'\} \in \mathcal{C}$, for each labeled reward machine model $\hat{\mathcal{G}}_i \in \mathcal{P}_{\mathrm{active}}$ with transition function $\hat \delta_\mathbf{u}^i$ and labeling function $\hat L^i$, let $ u_i = \hat \delta_\mathbf{u}^i(u_I,\hat L^i(\tau )) , 
    u_i' = \hat \delta_\mathbf{u}^i(u_I,\hat L^i(\tau')),$ and define the \emph{quality} of $\{\tau,\tau'\}$  to be:
\begin{equation}
\texttt{quality}(\tau,\tau') = \min \left\{ \sum_{i=1}^{N_{\text{active}}} \mathbf{1}(u_i = u_i'), \sum_{i=1}^{N_{\text{active}}} \mathbf{1}(u_i \neq u_i') \right\},
\end{equation}
which counts how many models in $\mathcal{P}_{\mathrm{active}}$ predict a node collapse ($u_i = u_i'$) versus a node separation ($u_i \neq u_i'$) when traversed with $\{\tau,\tau'\}.$
\item \textbf{Query:} Sort the trajectory pairs based on the quality metric and query the history policy for the top $\mathfrak{B}$ pairs. If $\exists s^*,a^*$ such that $\hpi(a^*|s^*,\tau) \neq \hpi(a^*|s^*,\tau')$, add $\{\tau,\tau'\}$ to the set of negative examples.
\item \textbf{Refine:} Resolve the SAT problem incrementally with the new negative examples and update the feasible solution set $\mathcal{P}_{\mathrm{feasible}}$. 
\item \textbf{Iterate/Terminate:} If the feasible solution set has converged to a single model up-to-renaming, terminate; otherwise, increment $l = l+1$ and return to Step 2.
\end{enumerate}
\end{myalgorithm}
 We note that the quality measure for a trajectory pair, $\texttt{quality}(\tau, \tau')$, is maximized when the pair bisects the candidate solution set. Also, given that exhaustive enumeration of all possible paths in Step 3 is computationally prohibitive, we employ a randomized Depth First Search (DFS) algorithm \cite{motwani1996randomized} with an upper bound on the size of $\mathcal{C}$\footnote{In our experiments, we set $|\mathcal{C}| \leq 10,000.$}. This allows the exploration of deeper paths within the history policy tree than exhaustive search allows. Another algorithmic optimization we employ is re-solving the SAT problem incrementally by simply adding the newly discovered negative examples to the existing SAT instance. This allows us to increment the depth in Step 7 of the algorithm without resolving with the exhaustive history policy (essentially the starting burn-in depth is fixed). We evaluate the empirical performance of this active approach in the experiments section.

\section{Experiments}
To evaluate the effectiveness of our proposed framework, we consider a grid world navigation environment (Figure~\ref{fig:gridworldroom}). Each cell in this $4 \times 4$ grid structure represents an MDP state. During navigation, the robot occasionally slips into neighboring cells upon taking a step along one of the 4 cardinal directions. We consider two tasks: \texttt{pick\_n\_drop} and  \texttt{patrolABCD}. Figures~\ref{fig:pick_place_rooms} and~\ref{fig:gridworldroom} show the color-coded ground-truth labeling function. For example, the blue cell in Figure~\ref{fig:pick_place_rooms} denotes the pickup location and the red cells in Figure~\ref{fig:gridworldroom} represents proposition $\mathrm{A}$. These labels ($\mathrm{A,B,C}, \text{ etc.}$) could in principle denote an area that has certain properties (cold/hot), or contains landmarks (coffee/mail). It is important to emphasize that the labeling of these cells is hidden from our algorithm. 
\vspace{-0.5cm}

\begin{figure}[h]
    \centering
    \begin{subfigure}[t]{0.28\textwidth}
        \centering
        \includegraphics[width=\textwidth]{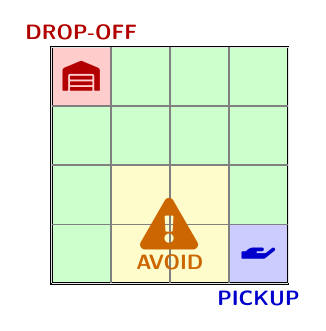}
        \caption{}
        \label{fig:pick_place_rooms}
    \end{subfigure}%
    \hspace{0.02\textwidth}%
    \begin{subfigure}[t]{0.34\textwidth}
        \centering
        \includegraphics[width=\textwidth]{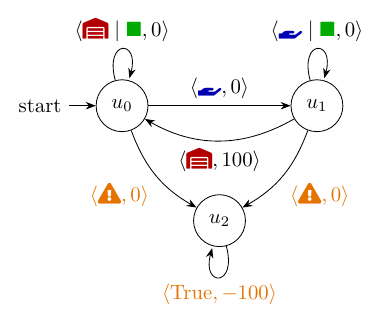}
        \caption{}
        \label{fig:pick_place_rm}
    \end{subfigure}%
    \hspace{0.02\textwidth}%
    \begin{subfigure}[t]{0.31\textwidth}
        \centering
        \resizebox{\textwidth}{!}{%
       \begin{tikzpicture}
    \begin{axis}[
        xlabel={Depth},
        ylabel={\# of Solutions},
        ymode=log,
        xmin=2.5, xmax=22,
        ymin=0.8, ymax=25000,
        xtick={4,6,8,10,12,14,16,18,20},
        extra y ticks={12},
        extra y tick labels={12},
        extra y tick style={
            grid=major,
            tick label style={color=red!80!black, font=\tiny},
            grid style={dashed, red!60!white}
        },
        grid=major,
        grid style={dashed, gray!30},
        legend style={at={(0.95,0.85)}, anchor=north east, font=\small},
    ]
    
    \addplot[color=blue!60!black, mark=square*, line width=1pt]
    coordinates {
        (4, 10430.24) (6, 10378.06) (8, 116.47) (9, 68.12) (10, 34.59) (12, 23.29) (14, 23.29) (16, 17.65) (18, 12.00) (20, 12.00)
    };
    \addlegendentry{$N_{\mathrm{active}}= 50$}

    \addplot[fill=green!60!black, fill opacity=0.15, draw=none, forget plot]
    coordinates {
        (4, 10633.3) (6, 10772.2) (8, 3235.8) (9, 101.8) (10, 39.9) (12, 12.0) (18, 12.0)
        (18, 12.0)   (12, 12.0)   (10, 4.8)    (9, 11.0)   (8, 1.0)    (6, 10082.9) (4, 10070.7)
    } \closedcycle;

    \addplot[color=green!60!black, mark=triangle*, line width=1pt]
    coordinates {
        (4, 10352.00) (6, 10427.60) (8, 733.87) (9, 56.40) (10, 22.40) (12, 12.00) (14, 12.00) (16, 12.00) (18, 12.00) (20, 12.00)
    };
    \addlegendentry{$N_{\mathrm{active}}= 100$}
    \addplot[fill=gray, fill opacity=0.15, draw=none, forget plot]
    coordinates {
        (4, 10627.89) (6, 10448.36) (8, 10673.61) (9, 10786.33) (10, 10758.70) (12, 10665.50) (14, 10705.23) (18, 10792.44)
        (18, 10236.41) (14, 10031.05) (12, 10082.22) (10, 10200.73) (9, 10209.95) (8, 9979.53) (6, 9947.92) (4, 10314.68)
    } \closedcycle;
    \addplot[color=gray, dashed, mark=x, line width=1pt]
    coordinates {
        (4, 10471.29) (6, 10198.14) (8, 10326.57) (9, 10498.14) (10, 10479.71) (12, 10373.86) (14, 10368.14) (18, 10514.43) (20, 10514.43)
    };
    \addlegendentry{Random}

    \end{axis}
\end{tikzpicture}
        }
        \caption{}
        \label{fig:solution_count_plot}
    \end{subfigure}
    \caption{(a) The warehouse grid world. (b) The pick and drop reward machine. (c) Solution count at increasing depths. The shaded area represents $\pm$ one standard deviation. Negative region is cut-off.}
    \label{fig:combined_figure}
\end{figure}
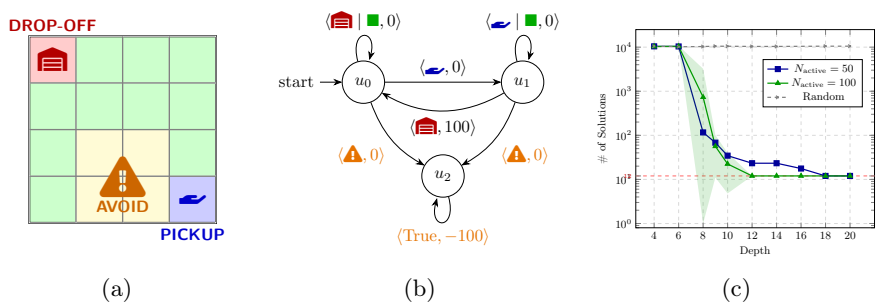
\vspace{-1cm}
\subsection{Task 1: \texttt{pick\_n\_drop}}
The robot’s objective is to perform a standard warehouse automation task: visit the pickup location (bottom right, Figure~\ref{fig:pick_place_rooms}) and then the drop-off location (top left, Figure~\ref{fig:pick_place_rooms}) in a cyclic manner while avoiding passing through a danger zone (the $4$ states colored yellow in the bottom middle, Figure~\ref{fig:pick_place_rooms}). The corresponding ground-truth reward machine is shown in Figure~\ref{fig:pick_place_rm}. Above each edge between two nodes, there is a tuple showing the label initiating the transition and the corresponding reward value. Critically, our learning algorithm does not have access to the reward machine transitions or the underlying warehouse arrangement. It operates solely on raw state trajectories extracted from the expert's patrolling policy. Using a depth-9 expert history policy, our algorithm returns $12$ solutions and successfully recovers the ground-truth reward machine. It correctly assigns distinct labels to the pickup, drop-off and avoid locations while grouping all remaining states under a common label. These solutions differ only by renaming, which does not change the task specification, making them equivalent to the ground truth. To evaluate the performance of our active extension algorithm, we start with a burn-in depth of $l = 3$. Due to the sparsity of negative examples at this depth, the number of feasible labeled reward machine models (i.e. feasible solutions to Problem~\ref{prob:SAT}) exceeds $150K$. We only keep $10K$ of these solutions. We run our active extension algorithm with $N_{\mathrm{active}} = \{50,100\}$ and a budget $\mathfrak{B} = 250$. We compare our active extension algorithm against a baseline that randomly generates feasible state trajectory pairs and queries the history policy. Results in Figure~\ref{fig:solution_count_plot} show the mean solution count across 15 independent trials. With $N_{\mathrm{active}} = 100$, our active extension algorithm converges to the ground-truth solution set in $100\%$ of the trials by depth $12$. Running our algorithm with $N_{\mathrm{active}} = 50$ also performed well (stabilizing at the ground truth solution by depth $18$), while the random baseline failed to find any restrictive constraints even as far as depth 20.\footnote{We note that any depth beyond $10$ is practically infeasible to solve using the full restriction of the history policy, highlighting the computational significance of our algorithm.}

\subsection{Task 2: \texttt{patrolABCD}}
The robot’s objective here is to patrol the rooms in the order $\mathrm{A} \to \mathrm{B} \to \mathrm{C} \to \mathrm{D}$ (Figure~\ref{fig:gridworldroom}). The task is encoded by the ground truth reward machine shown in Figure~\ref{fig:patrol_rm}. By using the depth-$9$ restriction of the expert's history policy, we recover the ground truth labeled reward machine model and the clustering of grid cells into their respective propositions up-to-renaming (this amounts to a total of $\mathbf{36}$ solutions\footnote{There is $6$ possible node naming permutations of the labeled reward machine model and $6$ naming permutations of the labeling assignments given that the label of the first state is anchored (Equation~\ref{eq:SAT:anchoring}), thus we have $6\times6 = 36$ total renaming solutions.}). As shown in Table~\ref{tab:resource_comparison}, this results in $|\mathcal{E}^-_{9}| \approx 414\text{M}$, meaning that we have over $414\text{M}$ negative examples. In our experiments, we group these negative examples by their terminal state and sample $5000$ of these negative examples at random from each group.  We also test our framework on a tetris variant of the room structures (Figure~\ref{gridworldroomtetris}), to which the results remain unchanged. That is, our algorithm perfectly recovers the labeling function up to renaming.

\begin{figure}[h]
    \centering
    \begin{subfigure}[t]{0.25\textwidth}
        \centering
        \includegraphics[width=\textwidth]{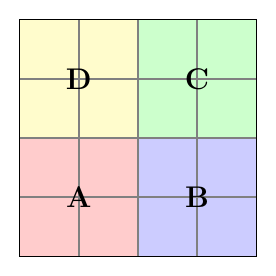}
        \caption{}
        \label{fig:gridworldroom}
    \end{subfigure}%
    \hspace{0.02\textwidth}%
    \begin{subfigure}[t]{0.4\textwidth}
        \centering
        \includegraphics[width=\textwidth]{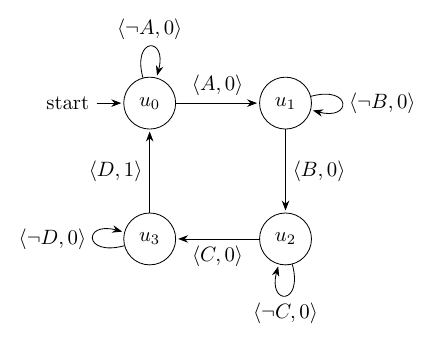}
        \caption{}
        \label{fig:patrol_rm}
    \end{subfigure}
    \hspace{0.02\textwidth}%
    \begin{subfigure}[t]{0.25\textwidth}
        \centering
        \includegraphics[width=\textwidth]{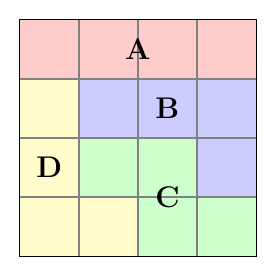}
        \caption{}
        \label{gridworldroomtetris}
    \end{subfigure}%
    \caption{(a) The room grid world. (b) The patrol reward machine. (c) The Tetris rooms grid world.}
    \label{fig:overall}
\end{figure}

For our active learning algorithm, we initialize the process with a burn-in depth of $l = 6$. At this depth, the history policy consists of $6895$ unique branches, and solving $\text{SAT}_{l=6}$ yields a hypothesis space of $1152$ distinct solutions. We constrain the negative example query budget to $\mathfrak{B} = 250$ per depth increment. We examine two sub-sampling sizes, $N_{\mathrm{active}} \in \{100,200\}$. Figure~\ref{fig:solution_count} shows the mean solution count across $30$ independent trials. The blue curve represents the solution count when using the full depth-$l$ restriction of the history policy, for $6 \leq l \leq 13$. Both $N_{\mathrm{active}}= 100$ and $N_{\mathrm{active}}= 200$ exhibit similar performance, achieving faster convergence to the ground truth solution set than the random sampling baseline. Specifically, with $N_{\mathrm{active}} = 200$, $96.6\%$ of trials converge to the ground-truth solution set (up-to-renaming) by depth $13$, while $83.3\%$ converge at the same depth with $N_{\mathrm{active}} = 100$. In contrast, the random sampling baseline still averages $378.0$ candidate solutions and exhibits a standard deviation nearly two orders of magnitude larger than that of our active extension method. 

\begin{figure}[ht]
    \centering
    \begin{subfigure}{0.49\textwidth}
        \centering
        \begin{tikzpicture}[scale=0.9,spy using outlines={rectangle,magnification=4.5,size=1.5cm,connect spies}]
\begin{axis}[
    width=\textwidth,
    title={\textbf{Solution Count}},
    xlabel={Depth},
    ylabel={\# of Solutions},
    xmin=5.5, xmax=13.5,
    ymin=-5, ymax=1450,
    xtick={6,7,8,9,10,11,12,13},
    ytick={400,800,1200},
    extra y ticks={36},
    extra y tick labels={36},
    extra y tick style={
        grid=major,
        tick label style={color=red!80!black, font=\tiny},
        grid style={dashed, red!60!white}
    },
    grid=major,
    grid style={dashed, gray!30},
    legend style={at={(0.95,0.95)}, anchor=north east, font=\tiny},
]

\addplot[color=blue!70!black, mark=*, line width=1pt,mark indices={0,1,2,3,4,5,6,7}]
coordinates {(6,1152)(7,144)(8,72)(9,36)(10,36)(11,36)(12,36)(13,36)};
\addlegendentry{Full $\hpi$}


\addplot[fill=orange, fill opacity=0.3, draw=none, forget plot]
coordinates {
    (6, 1152) 
    (7, 212.15) (8, 100.80) (9, 72.00) (10, 72.00) (11, 69.44) (12, 59.63) (13, 43.66) 
    (13, 30.74) (12, 29.17) (11, 33.76) (10, 72.00) (9, 72.00) (8, 57.60) (7, 114.25) 
    (6, 1152)
} \closedcycle;

\addplot[color=orange, mark=square*, line width=1pt, mark indices={0,1,2,3,4,5,6,7}]
coordinates {
    (6, 1152) 
    (7, 163.20) 
    (8, 79.20) 
    (9, 72.00) 
    (10, 72.00) 
    (11, 51.60) 
    (12, 44.40) 
    (13, 37.20)
};
\addlegendentry{$N_{\mathrm{active}} = 200$}

\addplot[fill=green!60!black, fill opacity=0.2, draw=none, forget plot]
coordinates {
    (6, 1152) 
    (7, 272.97) (8, 87.32) (9, 72.00) (10, 72.00) (11, 76.15) (12, 66.55) (13, 55.42) 
    (13, 28.58) (12, 31.85) (11, 41.45) (10, 72.00) (9, 72.00) (8, 61.48) (7, 91.83)  
    (6, 1152)
} \closedcycle;

\addplot[color=green!60!black, mark=triangle*, line width=1pt, mark indices={0,1,2,3,4,5,6,7}]
coordinates {
    (6, 1152) 
    (7, 182.40) 
    (8, 74.40) 
    (9, 72.00) 
    (10, 72.00) 
    (11, 58.80) 
    (12, 49.20) 
    (13, 42.00)
};
\addlegendentry{$N_{\mathrm{active}} = 100$}

\addplot[fill=gray, fill opacity=0.2, draw=none, forget plot]
coordinates {
     (6, 1152) (7, 1236.2) (8, 1245.7) (9, 1217.3) (10, 1099.8) (11, 995.9) (12, 785.7) (13, 666.1)
     (13, 89.9) (12, 155.1) (11, 261.7) (10, 407.4) (9, 549.1) (8, 923.9) (7, 1029.4) (6, 1152)
} \closedcycle;

\addplot[color=gray, mark=x, line width=1pt, mark indices={0,1,2,3,4,5,6,7,8}]
coordinates {
     (6, 1152) (7, 1132.8) (8, 1084.8) (9, 883.2) (10, 753.6) (11, 628.8) (12, 470.4) (13, 378.0)
};
\addlegendentry{Random}

\coordinate (target) at (axis cs:11.8,0);
\coordinate (viewer) at (axis cs:14,700); 

\spy on (target) in node [fill=white, draw, rounded corners] at (viewer);
\end{axis}
\end{tikzpicture}

        \caption{}
        \label{fig:solution_count}
    \end{subfigure}
    \begin{subfigure}{0.5\textwidth}
        \centering
        \begin{tikzpicture}[scale=0.75]
    \begin{semilogyaxis}[
        width=\textwidth,
        title={\textbf{Branching Complexity}},
        xlabel={Depth},
        ylabel={Number of Trajectories},
        xmin=5.5, xmax=13.5,
        grid=major,
        grid style={dashed, gray!30},
        legend pos=north west,
        ytick={10000, 1000000, 100000000},
        yticklabels={10K, 1M, 100M},
        log ticks with fixed point,
        legend style={font=\tiny},
        thick,
    ]
    
    \addplot[color=blue, mark=*, line width=1pt] coordinates {
         (6, 6895) (7, 26289) (8, 100247) (9, 382289) (10, 1529156)
    };
    \addlegendentry{Full $\hpi$}

    \addplot[color=blue, dashed, line width=1pt, forget plot, domain=10:13, samples=2] {1529156 * 4^(x-10)};

    \addplot[color=orange, mark=square*, line width=1pt] coordinates {
        (6, 6896) 
        (7, 7396) 
        (8, 7896) 
        (9, 8396) 
        (10, 8896) 
        (11, 9396) 
        (12, 9896) 
        (13, 10396)
    };
    \addlegendentry{Active - $\mathfrak{B}=250$}
    \end{semilogyaxis}
\end{tikzpicture}
        \caption{}
        \label{fig:branching_complexity}
    \end{subfigure}
    \caption{(a): Reduction in solution set size vs  depth. (b): Growth of the number of trajectories vs depth.}
    \label{fig:overall_comparison}
\end{figure}
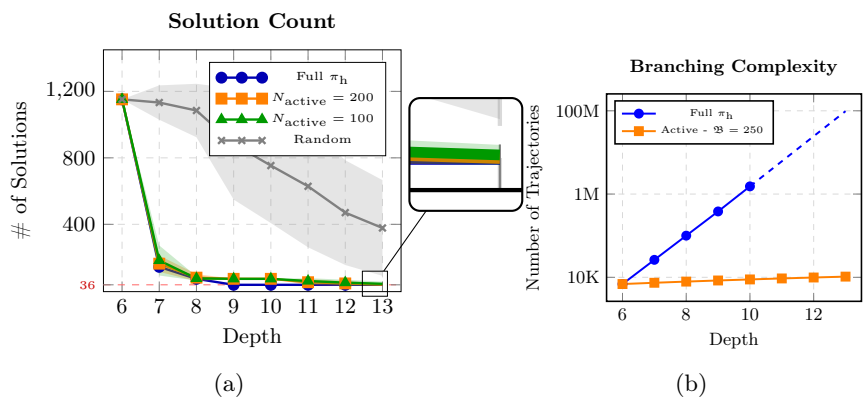
Beyond exact recovery, we show in Table~\ref{tab:resource_comparison} the efficiency gains of the active extension algorithm as compared to the full depth restriction history policy (exhaustive). The latter quickly hits a memory bottleneck, requiring approximately $24.76$ GB to store over $414$ M negative examples at depth $9$. This is due to the exponential growth in the number of state trajectories depicted in Figure~\ref{fig:branching_complexity}. In contrast, the active method selectively queries the environment and maintains only $10.3$ K branches and $0.292$ M negative examples even at depth $13$, reducing the memory required for negative examples to just $0.147$ GB, effectively reducing the memory requirement for negative examples by two orders of magnitude.
 
\begin{table}[ht]
    \centering
    \setlength{\tabcolsep}{5pt} 
    \renewcommand{\arraystretch}{1} 
    
    \begin{tabular}{lccccc}
        \toprule
        \textbf{Method} & \textbf{Depth ($l$)} & \textbf{$|\tau|$} & $|\mathcal{E}^{-}_l|$ & \textbf{size }(|$\tau$|) & \textbf{size }($|\mathcal{E}^{-}_l|$)   \\
        \midrule
        Exhaustive & 9 & 382K & 414M & 0.3Gb &24.76Gb  \\
        \addlinespace[0.5em] 
        Active Extension & 13  & 10.3K & 0.292M  & 0.046Gb & 0.147Gb  \\
        \bottomrule
    \end{tabular}
    \vspace{0.1cm}
    \caption{Memory Requirements for exhaustive vs. active search. $|\tau|$ is the total number of trajectories. $|\mathcal{E}^{-}_l|$ is the number of negative examples. \textbf{size }(|$\tau$|) is the memory required to store the trajectories. \textbf{size }($|\mathcal{E}^{-}_l|$) is the memory required to store the negative examples. }
    \label{tab:resource_comparison}
\end{table}

The active extension algorithm also yields significant runtime improvements. As shown in Table~\ref{tab:performance_comparison}, the exhaustive baseline is dominated by SAT solving, with a mean SAT time exceeding $7100$ seconds, despite terminating at depth $9$. In contrast, the active method substantially reduces the SAT burden to $2417.12$ seconds on average while scaling to a deeper horizon ($l=13$). Overall, the active extension achieves a mean total runtime of $3544.76$ seconds, nearly a $2\times$ speedup over the exhaustive baseline, while still recovering the full ground-truth solution set. These results  demonstrate that active extension alleviates both memory and computational bottlenecks, which lays the groundwork for scalable reward machine inference from raw state trajectories. 

\begin{table}[ht]
    \centering
    \begin{tabular}{@{}lccccc@{}}
        \toprule
        \textbf{Method} & \textbf{\shortstack{Mean Discovery \\ Time (s)}} & \textbf{\shortstack{Mean SAT \\ Time (s)}} & \textbf{\shortstack{Total \\ Time (s)}} & \textbf{\shortstack{Max \\ Depth}} & \textbf{Sols} \\
        \midrule
        Active Extension & $1127.64 \pm 186.11$ & $2417.12 \pm 6.16$& $3544.76\pm 188.71$ & 13 & 36 \\
        Exhaustive & 26.80 & $7158.73 \pm 1640.39 $ & $7185.53\pm 1640.39$ & 9 & 36 \\
        \bottomrule
    \end{tabular}
    \vspace{2pt}
    \vspace{0.1cm}
    \caption{Computation requirements for active extension vs. exhaustive Baseline. $\textbf{Mean Discovery Time (s)}$ refers to the time spent finding the negative example set. $\textbf{Mean SAT Time (s)}$ refers to the time spent solving a SAT problem instance. Results are over 10 separate trials. The mean SAT time for active extension includes $2350$ s spent solving SAT with the depth-$6$ policy.}
    \label{tab:performance_comparison}
\end{table}
\vspace{-1.5cm}
\section{Discussion and conclusion}

In this paper, we studied reward machine inference in an information-scarce setting where only raw state trajectories and a depth-limited history policy are available, without observing rewards, labels, or automaton nodes. We introduced a SAT-based formulation that jointly infers the reward machine transition structure and a labeling function, and established a sufficient depth condition under which additional history does not further constrain feasibility. Building on this, we proposed an active extension strategy that selectively queries informative trajectory pairs to reduce the hypothesis space efficiently, yielding substantial memory and runtime gains while still recovering the ground-truth solution set up to renaming whenever this is possible with the exhaustive version.

We emphasize that the present framework should be viewed as a foundational step, establishing the identifiability and algorithmic backbone of the problem, rather than a complete end-to-end practical solution. Because the current formulation assumes a discrete state-action space and access to the history policy induced by the optimal product policy, important questions remain regarding robustness to finite data and scalability to richer robotic domains. Nevertheless, the framework suggests several concrete directions for improving applicability: the history policy could be estimated from state-action trajectories using statistical extensions robust to estimation error as is done in known labeling function case  \cite{shehab2025learning}, while selective querying could substantially reduce the computational burden where exhaustive expansion is infeasible. Furthermore, extending the framework beyond the tabular setting will likely involve learning labeling functions directly from perceptual representations, potentially leveraging pre-trained models \cite{he2016deep}. A remaining limitation, however, is that the current termination criterion requires finding solutions up to renaming, which may be unnecessarily restrictive in practice. In general, it might be possible to terminate, even earlier that the sufficient depth, if all the candidate labeled reward machine models are policy-equivalent. Integrating such an equivalence-testing procedure into the active extension process while preserving the framework’s computational advantages represents a promising path for future research.

More broadly, our work fits within a larger question of why memory is needed in sequential decision-making. Arguably, there are two primary sources of such memory requirements. The first is partial observation or epistemic uncertainty, where memory is needed to construct a sufficient information state for decision-making. The second is task structure, where successful behavior depends on remembering progress through a temporally extended objective, as in our setting. Our contribution is aimed at uncovering this second form of memory structure, represented as a reward machine and labeling function, directly from policy data. A natural next step, therefore, is to relate the present framework to the literature on information states \cite{sakcak2024mathematical} and learning Partially Observed Markov Decision Processes \cite{shaw2026toward}, with the broader goal of developing a unified understanding of memory requirements in robot decision-making.

\begin{credits}
\subsubsection{\ackname}This work is supported in part by ONR CLEVR-AI MURI (\#N00014- 21-1-2431). Antoine thanks his son Math{\'e}o for giving him the time to finish this article. LLMs (ChatGPT, Gemini) have been used to polish parts of the writing of this paper. The output of the LLM is then thoroughly examined by the authors to maintain consistency and accuracy.
\end{credits}
%
%
%
\vspace{-0.4cm}
\bibliography{references}
\end{document}